\documentclass{article}

\usepackage{microtype}
\usepackage{graphicx}
\usepackage{subfigure}
\usepackage{booktabs} 
\usepackage{float}
\usepackage{bm}

\usepackage{hyperref}

\usepackage{etoolbox}
\newtoggle{fullversion}
\toggletrue{fullversion}
\iftoggle{fullversion}{
\usepackage[accepted]{icml2024} 
}{
\usepackage{icml2024} 
}

\usepackage{amsmath}
\usepackage{amssymb}
\usepackage{mathtools}
\usepackage{amsthm}
\usepackage{nicefrac}

\usepackage[capitalize,noabbrev]{cleveref}

\theoremstyle{plain}
\newtheorem{theorem}{Theorem}[section]
\newtheorem{proposition}[theorem]{Proposition}
\newtheorem{lemma}[theorem]{Lemma}

\theoremstyle{definition}

\theoremstyle{remark}
\newtheorem{remark}[theorem]{Remark}

\usepackage{enumerate}
\usepackage{tikz}
\usetikzlibrary{shapes.geometric, arrows, positioning}

\usepackage[disable,textsize=tiny]{todonotes}
\setuptodonotes{inline}
\usepackage{xcolor}

\icmltitlerunning{A Coding-Theoretic Analysis of Hyperspherical Prototypical Learning Geometry}

\begin{document}

\twocolumn[
\icmltitle{A Coding-Theoretic Analysis of Hyperspherical Prototypical Learning Geometry}



\icmlsetsymbol{equal}{*}

\begin{icmlauthorlist}
\icmlauthor{Martin Lindstr\"{o}m}{kth}
\icmlauthor{Borja Rodr\'iguez-G\'alvez}{kth}
\icmlauthor{Ragnar Thobaben}{kth}
\icmlauthor{Mikael Skoglund}{kth}
\end{icmlauthorlist}

\icmlaffiliation{kth}{Division of Information Science and Engineering, KTH Royal Institute of Technology, Stockholm, Sweden}

\icmlcorrespondingauthor{Martin Lindström}{martin12@kth.se}
\icmlcorrespondingauthor{Ragnar Thobaben}{ragnart@kth.se}

\icmlkeywords{Representation Learning, Prototypical Learning, Hyperspherical Prototypical Learning, Coding Theory, Geometry}

\vskip 0.3in
\editorsListText
\vskip 0.3in
]



\printAffiliationsAndNotice{}  

\begin{abstract}
Hyperspherical Prototypical Learning (HPL) is a supervised approach to representation learning that designs class prototypes on the unit hypersphere. The prototypes bias the representations to class separation in a scale invariant and known geometry. Previous approaches to HPL have either of the following shortcomings: (i) they follow an unprincipled optimisation procedure; or (ii) they are theoretically sound, but are constrained to only one possible latent dimension. In this paper, we address both shortcomings. To address (i), we present a principled optimisation procedure whose solution we show is optimal. To address (ii), we construct well-separated prototypes in a wide range of dimensions using linear block codes. Additionally, we give a full characterisation of the optimal prototype placement in terms of achievable and converse bounds, showing that our proposed methods are near-optimal.

GitHub: \href{https://github.com/martinlindstrom/coding_theoretic_hpl}{martinlindstrom/coding\_theoretic\_hpl}
\end{abstract}
\section{Introduction}\label{sec:introduction}
\looseness=-1 Representation learning addresses the problem of learning a mapping from a high-dimensional input space to a lower-dimensional representation space subject to suitable inductive biases. These biases are imposed on learning algorithms as additional constraints on, for example, the network architecture or the optimisation algorithm \citep{goyal_inductive_2022}. Geometry-based inductive biases have long been popular in representation learning. For instance, imposing unit norm constraints on the representations has been employed in different unsupervised learning methods, either through explicit normalisation or norm-invariant loss functions in variational autoencoders \citep{davidson_hyperspherical_2018, xu_spherical_2018}, or in self-supervised learning \citep{wang_understanding_2020, wang_normface_2017, wu_unsupervised_2018, chen_simple_2020, bachman_learning_2019, caron_unsupervised_2020}. Imposing representation separation is another common inductive bias used, for example, in contrastive learning \citep{wang_understanding_2020, galvez_role_2023, chen_simple_2020, tian_contrastive_2020} and supervised representation learning \citep{khosla_supervised_2020, guerriero_deepncm_2018, hasnat_von_2017}.

\looseness=-1 In the supervised learning setting, one way to impose representation separation is through \emph{prototypical learning}~\citep{snell_prototypical_2017, jetley_prototypical_2015}. Each class is assigned a prototype, and these are  specified \emph{a priori} to maximise their separation and are held fixed during training, where the algorithm attempts to map input samples to their class prototypes. Therefore, the representations are biased towards being separated based on their class. Recently, Hyperspherical Prototypical Learning (HPL) \citep{mettes_hyperspherical_2019, kasarla_maximum_2022} started imposing unit norm constraints to prototypical learning. This places the representations on the hypersphere and thus enhances representation separation bias in a scale invariant and known geometry.

\begin{figure*}[h]
	\centering
	\includegraphics[width=\linewidth, clip, trim=0.6cm 1.3cm 0.1cm 2.4cm]{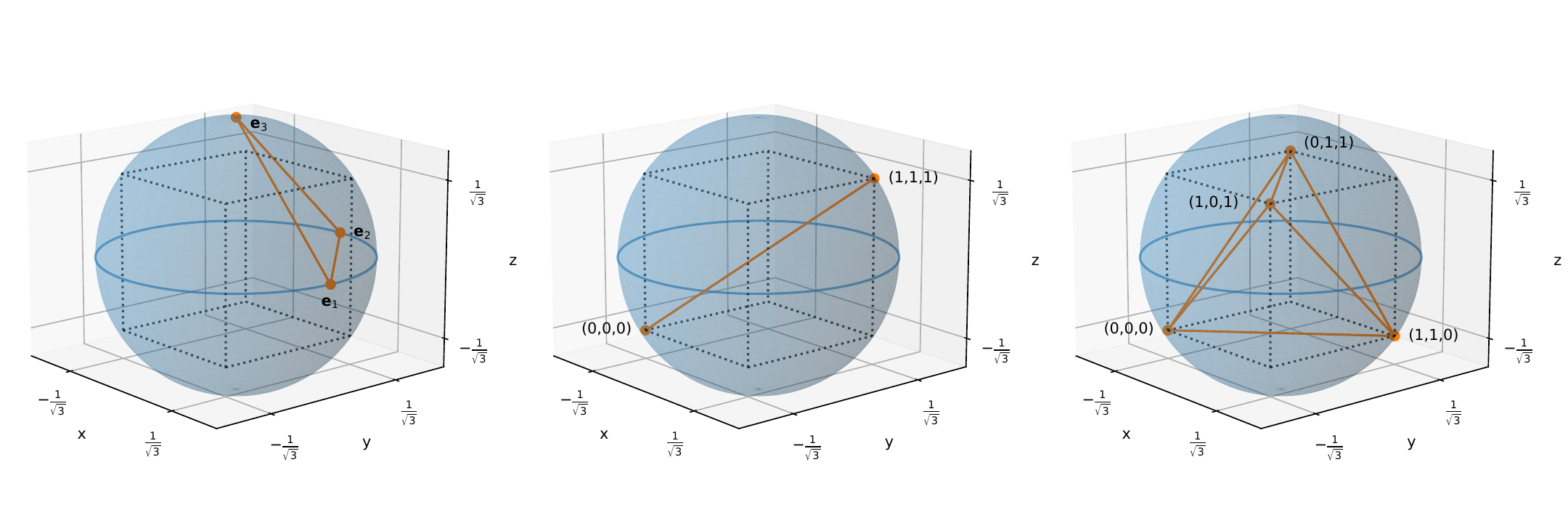}
	\caption{Prototypes on $\mathbb{S}^2$. The \textbf{left} image shows the naïve one-hot encoding approach, which has cosine similarity $0$. One can do much better with binary code-based prototypes. In the \textbf{centre} image, by reducing to two prototypes, the optimal distance with cosine similarity $-1$ is achieved. The \textbf{right} image shows how one can fit four prototypes with a better separation, arriving at a cosine similarity of $\nicefrac{-1}{3}$.}
	\label{fig:cube_sphere}
\end{figure*}

\looseness=-1 To illustrate the idea behind HPL, consider the naïve approach of selecting prototypes as the familiar one-hot encoding that, for $K$ classes, picks the canonical basis vectors $\{ \bm{e}_1, \ldots, \bm{e}_K \}$ in dimension $K$ as prototypes. As illustrated in \Cref{fig:cube_sphere} (left), this results in a suboptimal class separation on the hypersphere. Instead, HPL attempts at placing $K$ maximally separated prototypes on the $n$-dimensional hypersphere $\mathbb{S}^{n-1}$, hopefully with $n < K$. This combinatorial and non-convex problem is well-studied \citep{conway_sphere_1999}, but even on $\mathbb{S}^{2}$ the problem is unsolved for general $K$, and optimal solutions are only known for $K=1, \dots, 14$, and $24$ \citep{musin_tammes_2015}. Despite this, approximate solutions have been proposed. \citet{mettes_hyperspherical_2019} propose a relaxation of the problem that however only achieves sub-optimal separation. \citet{kasarla_maximum_2022}, on the other hand, propose a closed-form solution that we show to be optimal; however, it is only applicable in dimension $n=K-1$.

\looseness=-1 In this paper, we propose two new methods for designing hypershperical prototypes and present sharp bounds on the optimal separation that can be achieved by placing an arbitrary number of prototypes $K\le2^n$ on a hypershpere of dimension $n$. Our approach rests on theory and concepts from error correcting codes, and our contributions are threefold:
\begin{enumerate}[(i)]
    \item We provide a new design approach for  hyperspherical prototypes that maps binary linear codes defined over the $n$-dimensional Hamming space onto the  $n$-dimensional hypersphere $\mathbb{S}^{n-1}$. Our approach provides guarantees on the  class separation by design, at the same time that it enables a more flexible trade-off between separation and the dimension $n$ for a given number of classes $K$.

    \item We derive a \emph{converse} bound on the guaranteed minimum prototype separation as well as an \emph{achievable} bound that certifies that well-separated code-based prototypes exist. These bounds imply that for a large number of classes $K$ and in high dimensions $n$, the worst-case cosine similarity converges to zero. The bounds also show that our code-based prototypes closely approach optimal separation for $n \approx \nicefrac{K}{2}$.

    \item Finally, we provide alternative optimization-based hyperspherical prototypes which achieve the converse bound through a convex relaxation. These improve on the prototypes obtained by~\citet{mettes_hyperspherical_2019}, which do not achieve the converse bound.
\end{enumerate}

The paper is organised as follows. In \Cref{sec:primer}, we motivate the connection between binary codes and hyperspherical prototypes and give a primer on the theory of error correcting codes. In \Cref{sec:prototypes}, we apply this theory in order to give both coding-theoretic prototype constructions and bounds thereon. Also in \Cref{sec:prototypes}, we provide a novel optimisation-based prototype scheme. The performance of the proposed schemes is evaluated and compared to the state of the art in HPL in Section~\ref{sec:experiments}, and Section~\ref{sec:conclusion} concludes the paper.   
\section{Background}
\label{sec:primer}
We start with a formal problem formulation for designing a codebook of hyperspherical prototypes in the HPL setting. Then, we connect binary error correcting codes defined in the Hamming space to hyperspherical prototypes. After that, we provide a brief overview of fundamental concepts in coding theory that are used in this paper to derive code-based hyperspherical prototypes with good separation properties. The interested reader is encouraged to consult \citet{macwilliams_theory_1977} for a more detailed treatment.
\subsection{Problem Formulation}\label{sec:Problem}
We consider  the HPL setting with $K$ classes and $n$ dimensions; that is, we  are interested in placing $K$ prototypes $\bm{c}_1, \dots, \bm{c}_K$  on the $n$-dimensional unit hypersphere $\mathbb{S}^{n-1}$, where the dimension $n$ is a hyperparameter. Our objective is maxmising the Euclidean distance between every pair of prototypes $\bm{c}_i, \bm{c}_j\in\mathbb{S}^{n-1}$ ($i\ne j$). Clearly, the Euclidean distance $d_\mathrm{E}(\bm{c}_i,\bm{c}_j)$ is bounded in the range $[0,2]$ and satisfies that  $\|\bm{c}_i-\bm{c}_j\|^2 = 2 - 2\langle \bm{c}_i, \bm{c}_j \rangle$ for all $\bm{c}_i, \bm{c}_j\in\mathbb{S}^{n-1}$. Hence, maximising the Euclidean distance is equivalent to minimising the cosine similarity $\langle \bm{c}_i, \bm{c}_j \rangle$. In turn, this is equivalent to maximising the angle $\alpha$ between $\bm{c}_i$ and $\bm{c}_j$ since $\alpha = \arccos{\langle \bm{c}_i, \bm{c}_j \rangle}$. Therefore, we will use the cosine similarity as a notion of separation throughout this paper. The objective in HPL is then designing a codebook $\mathcal{C} := \{\bm{c}_i \in \mathbb{S}^{n-1} : i = 1, \dots, K\}$ of $K$ well-separated hyperspherical  prototypes, which can be summarized in the following optimisation problem: 
\begin{equation}\label{eq:optimal_hypershperical_prototypes}
\begin{aligned}
    \min_{\mathcal{C}} \max_{i \ne j} \ \langle \bm{c}_i, \bm{c}_j \rangle. 
\end{aligned}
\tag{$\mathrm{P}$}
\end{equation}
Unfortunately, this problem is both non-convex due to the unit norm constraint, and combinatorial due to the search of the worst pair of prototypes $\bm{c}_i, \bm{c}_j$ ($i\ne j$), requiring tractable relaxations that yield approximate solutions.
\subsection{Connecting Codes to Protypes}\label{sec:intuition}
The approach in this paper leverages coding theory to design $n$-dimensional binary vectors (that is, members of the $n$-dimensional Hamming space), which are mapped onto the $n$-dimensional hypersphere $\mathbb{S}^{n-1}$, thereby creating prototypes with good separation. To provide some intuition as to how error correcting codes relate to placing prototypes that are maximally spaced apart, consider the mapping $\pi : \{ 0, 1 \}^n \to \mathbb{S}^{n-1}$ that maps $n$-dimensional binary vectors $\bm{b}$ from the $n$-dimensional Hamming space to points $\bm{c}$ on the hypersphere. More precisely, the mapping is defined as 
\begin{equation}
\bm{c} = \pi(\bm{b}) \coloneqq \frac{2 (\bm{b} - \nicefrac{1}{2})}{\sqrt{n}}
\label{eq:mapping}
\end{equation}
and enforces that $\bm{c}(\ell) \in \{ - \nicefrac{1}{\sqrt{n}}, +\nicefrac{1}{\sqrt{n}} \}$ and $\|\bm{c}\|^2=1$. This approach allows us to place $2^n$ points on the unit hypersphere $\mathbb{S}^{n-1}$ with a cosine similarity of at most  $\langle \bm{c}, \bm{c}' \rangle \le 1 - \nicefrac{2}{n}$ 
for every pair $\bm{c} = \pi(\bm{b})$ and $\bm{c}' = \pi(\bm{b}')$ with $\bm{b} \neq \bm{b}'$. For $K<2^n$, we can improve the separation guarantees by carefully selecting the $K$ binary vectors placed on the hypersphere via the mapping $\pi$. Error correcting codes provide a systematic way to achieve this, and the concept is illustrated for $n=3$ in  \Cref{fig:cube_sphere}.   

As a baseline, consider one-hot encoding (\Cref{fig:cube_sphere}, left), which provides $K=3$ orthogonal prototypes. Using carefully selected corners of the unit cube and the mapping $\pi$, we can substantially improve on one-hot encoding. Reducing to $K=2$ prototypes (\Cref{fig:cube_sphere}, centre), the unit cube corners corresponding to $\bm{b}_1=(0,0,0)$ and $\bm{b}_2=(1,1,1)$ are diametrically opposed with a cosine similarity of $-1$, which is optimal. From a coding theory perspective, this corresponds to a repetition code with \emph{Hamming distance} $d_\mathrm{H}(\bm{b}_1,\bm{b}_2)=3$ between codewords; that is, the codewords differ in $3$ bits. Increasing to $K=4$ prototypes (\Cref{fig:cube_sphere}, right), the unit cube corners corresponding to $\bm{b}_1=(0,0,0)$, $\bm{b}_2=(0,1,1)$, $\bm{b}_3=(1,0,1)$, and $\bm{b}_4=(1,1,0)$  have a mutual cosine similarity of $\nicefrac{-1}{3}$, which is an improvement over one-hot encoding (\Cref{fig:cube_sphere}, left) while  increasing the number of classes at the same time. Again, the set $\mathcal{B}=\{\bm{b}_1, \bm{b}_2, \bm{b}_3, \bm{b}_4 \}$ constitutes a binary linear code with Hamming distance $d_\mathrm{H}(\bm{b}_i,\bm{b}_j)=2$ between its codewords.

This example demonstrates that due to the definition of the mapping function $\pi$, there exists a relation between the separation of vectors $ \bm{c}=\pi(\bm{b})$ and $ \bm{c}'=\pi(\bm{b}')$ on the unit hypersphere and the Hamming distance $d_\mathrm{H}(\bm{b}, \bm{b}')$ of the binary vectors $\bm{b}$ and $\bm{b}'$ in the sense that a large Hamming distance implies a large separation. We will make this result explicit in \Cref{sec:binary_to_euclidian}. Error correcting codes, designed to have a large Hamming distance $d_\mathrm{H}(\bm{b}, \bm{b}')$ between every pair of codewords $\bm{b}$ and $\bm{b}'$, are hence a well suited tool for designing hyperspherical prototypes, and coding theory provides us with useful bounds on the achievable separation. 
\subsection{A Primer on Coding Theory}
The systematic study of error correcting codes dates back to the seminal work of \citet{hamming_error_1950}. By introducing redundancy in a structured way, error correcting codes allow for error detection  and correction in messages and data, and are essential for guaranteeing the reliability of today's digital communication, computation, and storage systems. Linear codes defined over the Galois field $\mathrm{GF}(q)$, where $q=p^m$ and $p$ is a prime, are of special interest as they offer a structure that can be used for efficient encoding, decoding, and analysis of the distance properties of the code. 

In this paper, we mainly restrict ourselves to binary linear block codes (that is,  $q=p=2$), and only briefly discuss extensions to $q$-ary codes with $q>2$. A binary block code with parameters $[n,k]$ is specified by a codebook $\mathcal{B}$ of $2^k$  binary codewords of length $n$ and a bijective encoder mapping that maps the set of all length-$k$ binary vectors into the code $\mathcal{B}$. This adds $n-k$ bits of redundancy, which we can use to detect and correct errors in the codeword. The \emph{rate} of the code is defined as $R = \nicefrac{k}{n}$, where a low rate corresponds to a high redundancy. The error detection and correction capabilities rely on the \emph{separation} of codeword pairs $\bm{b}_i, \bm{b}_j \in \mathcal{B}$ in Hamming distance, namely
\begin{equation*}
    d_{\mathrm{H}}(\bm{b}_i, \bm{b}_j) := \sum_{\ell = 1}^{n} \mathbb{I}\{ \bm{b}_i(\ell) \ne \bm{b}_j(\ell)\}.
\end{equation*}
A binary code with minimum Hamming distance
\begin{equation*}
    d_\mathrm{min} := \min_{\bm{b}_i, \bm{b}_j \in \mathcal{B}, i \ne j} d_{\mathrm{H}}(\bm{b}_i, \bm{b}_j)
\end{equation*}
can detect $d_\mathrm{min}-1$ errors and correct $\lfloor \nicefrac{(d_\mathrm{min}-1)}{2}\rfloor$ errors. 

In this paper, we fix $k=\lceil \log_2(K) \rceil$ given $K$ classes, suggesting that we are in the low-rate or high-redundancy regime if $n$ is of the order of $K$. Noting that low-rate codes  offer large separation in terms of minimum Hamming distance, we can expect good separation by adopting a code-based approach. To this end, a fundamental result in coding theory is that \emph{good codes exist}. This is formalised by the well-known Gilbert-Varshamov bound \citep[Chapter 1, Theorem 12]{macwilliams_theory_1977}.
\begin{lemma}[Gilbert-Varshamov Bound]\label{lemma:gilbert-varshamov}
There exists an $[n,k]$ code with minimum distance at least $d_\mathrm{min}$, provided that
\begin{equation}\label{eq:gilbert-varshamov}
	2^{n-k} > \sum_{i=0}^{d_\mathrm{min}-2} \binom{n-1}{i}  .
\end{equation}
\end{lemma}
The Gilbert-Varshamov bound for the largest $d_\mathrm{min}$ gives a lower bound on $d_\mathrm{min}$. However, the bound only guarantees that good codes \emph{exist}, but not how to find them. Luckily, as we will show in \Cref{sec:prototypes}, several linear binary codes with better minimum distance exist. 
\begin{remark}\label{remark:rewrite_gv}
To derive some of our results, the bound in \eqref{eq:gilbert-varshamov} needs to be evaluated carefully in order to avoid overflow problems. Notice that the bound can be rewritten as 

\begin{align*}
    2^{-k} &> 
    \sum_{i=0}^{d_\mathrm{min}-2} \binom{n-1}{i} \left(\frac{1}{2}\right)^{i} \left(\frac{1}{2}\right)^{n-i} \\
    &= F_\mathrm{bin}\left(d_\mathrm{min}-2; n-1, \tfrac{1}{2}\right),
\end{align*}
where $F_\mathrm{bin}(\cdot \ ; n-1, \nicefrac{1}{2})$ denotes the cumulative distribution function of a binomial distribution with $n-1$ trials with success probability $\nicefrac{1}{2}$. This is implemented in a numerically stable manner in many computational libraries.
\end{remark}
\section{Hyperspherical Prototype Design}\label{sec:prototypes}
In this section, we present our main contributions towards solving the optimization problem \eqref{eq:optimal_hypershperical_prototypes}. As mentioned earlier, this problem is both non-convex due to the unit norm constraint, and combinatorial due to the search of the worst pair of prototypes $\bm{c}_i$ and $\bm{c}_j$ with $i\ne j$. Our contributions are focused on relaxations to the problem \eqref{eq:optimal_hypershperical_prototypes} and bounds on the optimal solution. \Cref{sec:binary_to_euclidian} formalises the coding-theoretic approach introduced with the example in \Cref{fig:cube_sphere}; \Cref{sec:bounds} uses coding-theoretic tools to bound the optimal solution to \eqref{eq:optimal_hypershperical_prototypes}; and \Cref{sec:optimisation_based} presents a relaxation to \eqref{eq:optimal_hypershperical_prototypes} which achieves the bound on the optimal solution.
\subsection{Coding-Theoretic Prototypes}\label{sec:binary_to_euclidian}
We begin by formalising the intuition provided in \Cref{sec:intuition}, namely that a binary $[n,k]$ code with the codebook $\mathcal{B}$ and a large minimum distance $d_\mathrm{min}$ gives good prototypes. Firstly, notice that if we want $K$ prototypes, then we need $|\mathcal{B}| = 2^k \ge K$. Additionally, recall that the mapping $\bm{c} = \pi(\bm{b}) $ defined in \eqref{eq:mapping} produces a unit-norm vector for any binary vector $\bm{b}$. Then, we can guarantee that binary code-based constructions guarantee good separation.
\begin{proposition}\label{prop:binary_to_euclidian}
Assume that $\mathcal{B}$ is the codebook of a binary $[n,k]$ code with minimum distance $d_\mathrm{min}$. Then, for every pair $\bm{b}_i,\bm{b}_j \in \mathcal{B}$ with $ i \ne j$, the cosine similarity between $\bm{c}_i=\pi(\bm{b}_i)$ and $\bm{c}_j=\pi(\bm{b}_j)$ is upper bounded by
\begin{equation*}
	\langle \bm{c}_i, \bm{c}_j \rangle = 1-\frac{2 d_{\mathrm{H}}(\bm{b}_i, \bm{b}_j)}{n} \le 1-\frac{2 d_\mathrm{min}}{n}.
\end{equation*}
\end{proposition}
\begin{proof}
To show the bound on the cosine similarity, notice that two binary vectors differing in $d_{\mathrm{H}}(\bm{b}_i, \bm{b}_j)$ positions obey that $\sum_{\ell=1}^n (\bm{b}_i(\ell ) - \bm{b}_j(\ell ))^2 = d_{\mathrm{H}}(\bm{b}_i, \bm{b}_j)$. Expanding $\|\bm{c}_i - \bm{c}_j\|^2$ in two ways gives
\begin{align*}
	\| \bm{c}_i - \bm{c}_j \|^2 &= 2-2\langle \bm{c}_i, \bm{c}_j \rangle \\
	&= \sum_{\ell=1}^n \left(\frac{2 (\bm{b}_i(\ell ) - \nicefrac{1}{2}) - 2(\bm{b}_j(\ell ) - \nicefrac{1}{2})}{\sqrt{n}}\right)^2.
\end{align*}
Rearranging this and recalling that $d_{\mathrm{H}}(\bm{b}_i, \bm{b}_j) \ge d_\mathrm{min}$ yields the desired result.
\end{proof}

\looseness=-1 Since good codes exist (see \Cref{lemma:gilbert-varshamov}), we provide two examples of binary codes whose minimum distance is close to $d_\mathrm{min} = \nicefrac{n}{2}$ in dimensions where $n < K$. That is, there exist no worse than orthogonal prototypes with zero worst-case cosine similarity in dimensions $n\approx \nicefrac{K}{2}$. Additionally, these codes are easy to implement in Python, and hence are easily integrable in modern machine learning software. The codes are the Bose--Chaudhuri--Hocquenghem (BCH) and Reed--Muller (RM) codes. Other code families like low-density parity-check (LDPC) codes and sparse graph codes  are not competitive in terms of minimum distance guarantees since their minimum distance is usually lower than the one predicted by the Gilbert-Varshamov bound, see \textit{e.g.}, \citet[Figure 9]{mitchell_spatially_coupled}. These popular code families are  hence not further considered in this paper. Polar codes, on the other hand, belong to the same code family as RM codes, see \textit{e.g.}, \citet[Section IV-D]{abbe_reedmuller_2021}, and are hence implicitly covered.

\paragraph{Prototypes from BCH Codes} BCH codes are known to have good minimum distance in low dimensions \citep[pp. 258]{macwilliams_theory_1977}, and although the exact minimum distance is not known in general \citep{li_minimum_2017}, we find empirically that it approaches $d_\mathrm{min} = \nicefrac{n}{2}$ in dimensions $n < K$. They are implemented in the \texttt{Galois} Python library \citep[v0.3.8]{hostetter_galois_2020}. 

\paragraph{Prototypes from RM Codes} The distance properties of RM codes are easier to characterise. Their construction is simple and gives straight-forward distance guarantees~\citep{abbe_reedmuller_2021}. In fact, they provide no worse than orthogonal prototypes in dimension $n \approx \nicefrac{K}{2}$.
\begin{lemma}[Separation Guarantees for  RM Codes]\label{lemma:orthogonal_rm_codes}
Let $\widetilde{K}$ be the smallest power of $2$ such that $\widetilde{K} \ge K$. Then, RM codes in dimension $n \ge \nicefrac{\widetilde{K}}{2}$ have minimum distance $d_\mathrm{min}=\nicefrac{n}{2}$ and guarantee that $\langle \bm{c}_i, \bm{c}_j \rangle \le 0$.
\end{lemma}
\begin{proof}
By construction, RM codes are $[n,k]$ codes with $n = 2^m$, $k = \sum_{i=0}^r \binom{m}{i}$, and minimum distance $d_\mathrm{min} = 2^{m-r}$. Hence, we have cosine similarity $\langle \bm{c}_i, \bm{c}_j \rangle \le 0$ if $r=1$, namely if $1+m = 1+ \log_{2} n = k \ge \log_{2} \widetilde{K} \ge \log_{2} K$ or if $n \ge \nicefrac{\widetilde{K}}{2} \ge \nicefrac{K}{2}$. 
\end{proof}
\begin{remark}
    It is important to note that the cosine similarity guarantee $\langle \bm{c}_i, \bm{c}_j \rangle \le 0$ is not equivalent to pairwaise orthogonality. However, every codeword is \emph{locally} orthogonal to all its minimum-distance neighbours and has no worse than orthogonal separation \emph{globally}. We investigate the global cosine similarity distribution for all prototype generation schemes in \Cref{fig:prototype_density}.
\end{remark}

\paragraph{Realisable Dimensions with Codes} As has been shown, binary codes provide a flexible way to derive prototypes. However, there is a restriction on the dimensions which are realisable: RM codes are defined for $n = 2^m$, and BCH codes are defined for $n = 2^m - 1$ for every $m \in \mathbb{N}_+$. Moreover, additional dimensions are realisable: codes can be \emph{punctured} by removing dimensions, thereby creating a lower-dimensional code, and they can be \emph{extended} by adding more dimensions \citep[Chapter 1, \S9]{macwilliams_theory_1977}. In general, puncturing a code by $1$ bit will reduce its minimum distance by $1$. Similarly, extending a code can (but is not guaranteed to) increase the minimum distance. Therefore, RM and BCH codes can have good distance properties around dimensions  $n=2^m$ and $n=2^m-1$, respectively, but not for general $n$. Compared to \citet{kasarla_maximum_2022}, which is only valid in $n=K-1$, coding-based prototypes hence improve flexibility by guaranteeing good separation for a larger set of admissible dimensions. 
\subsection{Coding-Theoretic Bounds}\label{sec:bounds}
In this section, we provide both upper and lower bounds the worst-case cosine similarity of hyperspherical prototypes in \eqref{eq:optimal_hypershperical_prototypes}. Our achievable (upper) bound is based on \Cref{lemma:gilbert-varshamov}, which states that good binary codes exist. Our converse (lower) bound is based on results from spherical coding theory, which shows that \emph{the minimum separation cannot be improved beyond near orthogonality}. We begin by recalling the Rankin bound from spherical coding theory \citep[Theorem 1.4.1]{ericson_codes_2001}.
\begin{lemma}[Rankin Bound]\label{lemma:rankin}
Any set of $K$ hyperspherical prototypes $\mathcal{C}$ satisfies that
\begin{equation*}
    \max_{\mathcal{C}} \min_{i\ne j} \|\bm{c}_i - \bm{c}_j\|^2 \le \frac{2K}{K-1}.
\end{equation*}
\end{lemma}
Now, we may state the achievable and converse bounds.
\begin{theorem}\label{thm:sandwich_thm}
There exists a set of hyperspherical prototypes $\mathcal{C}$ with cosine similarity at most (separation at least)
\begin{equation}
	\max_{i \ne j} \ \langle \bm{c}_i, \bm{c}_j \rangle \le 1 - \frac{2 d_{GV}}{n},
    \label{eq:achievable_bound}
\end{equation}
where $d_{GV}$ denotes the largest solution to the Gilbert-Varshamov bound in \eqref{eq:gilbert-varshamov}. Conversely, no set of prototypes exists with maximum cosine similarity smaller (better separation) than
\begin{equation}
	\max_{i \ne j} \ \langle \bm{c}_i, \bm{c}_j \rangle \ge \frac{-1}{K-1}.
    \label{eq:converse_bound}
\end{equation}
\end{theorem}
\begin{proof}
The achievable bound~\eqref{eq:achievable_bound} follows directly from combining \Cref{prop:binary_to_euclidian} and \Cref{lemma:gilbert-varshamov}. For the converse bound, recalling that $\|\bm{c}_i - \bm{c}_j\|^2 = 2 - 2 \langle \bm{c}_i, \bm{c}_j \rangle$, applying \Cref{lemma:rankin}, and simplifying yields~\eqref{eq:converse_bound}.
\end{proof}

The bounds are numerically evaluated in \Cref{sec:experiments}. A number of remarks are in order. In many practical settings, the converse bound $\nicefrac{-1}{(K-1)}$ is close to $0$. It is therefore impossible to achieve a maximum cosine similarity that is notably better than one-hot encoding. However, as \Cref{lemma:orthogonal_rm_codes} shows, it is possible to have no worse than orthogonal prototypes in low dimension $n=\nicefrac{\tilde{K}}{2}$. Moreover, as we will show with numerical examples, it is possible to have approximately orthogonal prototypes in much lower dimension than $n=K$. Finally, we note that the upper bound can be tightened for $n \ge K$ by recalling one-hot encoding. The bounds are therefore sharp, meaning that orthogonal prototypes are achievable and near-optimal for a large number of classes $K$. Finally, it is interesting to note that the mapping proposed by \citet{kasarla_maximum_2022} is \emph{optimal} since it achieves the converse bound. 
\subsection{Beyond Binary Codes}
In this section, we briefly discuss the generalisation of our results to the case of $q$-ary codes and  motivate the choice of restricting our attention to binary codes.

Assume a construction that combines an $[n_q,k_q]$ code $\mathcal{U}$ over the Galois field GF($q$), with minimum Hamming distance $d_{\mathrm{H, min}}^{(\mathcal{U})}$, with a mapping $\pi_{q}^{(l)}$ that maps $q$-ary symbols $u\in\{0,\ldots,q-1\}$ to points $\tilde{\bm{c}}=\pi_{q}^{(l)}(u)$ on the $l$-dimensional unit hypersphere $\mathbb{S}^{l-1}$, and with pairwise {Euclidean distance} of at least $d_{\mathrm{E, min}}^{(\pi_q^{(l)})}$. Then, an $n_q$-dimensional $q$-ary vector $\bm{u}$ can be mapped to the  unit hypersphere $\mathbb{S}^{n-1}$ in $n=n_q\cdot l$ dimensions by realising the mapping 
\begin{equation}
\label{eq:qary-mapping}
\bm{c} = \pi_q(\bm{u}) \coloneqq \frac{1}{\sqrt{n_q}}\left(\pi_{q}^{(l)}(\bm{u}(1)), \ldots,  \pi_{q}^{(l)}(\bm{u}(n_q)) \right).
\end{equation}
Then, similarly to the binary case, by a proper choice of the code parameters, hyperspherical prototypes with good separation guarantees can be obtained.

\begin{proposition}\label{prop:q-ary_to_euclidian}
Assume $\mathcal{U}$ is the codebook of a $q$-ary $[k_q,n_q]$ code with a minimum Hamming distance $d_{\mathrm{H, min}}^{(\mathcal{U})}$ that is mapped into the unit hyperspehere $\mathbb{S}^{n-1}$ in $n=n_q\cdot l$ dimensions with the mapping $\pi_q$ from~\eqref{eq:qary-mapping}. Furthermore, let $d_{\mathrm{E, min}}^{(\pi_q^{(l)})}$ be the minimum Euclidean distance achieved by the component mapping $\pi_q^{(l)}$. Then, for every codeword pair $\bm{u}_i,\bm{u}_j \in \mathcal{U}$ with $ i \ne j$, the cosine similarity between $\bm{c}_i$ and $\bm{c}_j$ is upper bounded by
\begin{equation*}
	\langle \bm{c}_i, \bm{c}_j \rangle \le 1-\frac{ d_{\mathrm{H, min}}^{(\mathcal{U})}}{n_q} \frac{\left[d_{\mathrm{E, min}}^{(\pi_q^{(l)})}\right]^2}{2}.
\end{equation*}
\end{proposition}
\begin{proof}
The proof follows along the same lines as the proof of Proposition~\ref{prop:binary_to_euclidian}.    
\end{proof}
Assume we want to minimise the upper bound on the cosine similarity. We then want to find a $q$-ary code with as large minimum distance as possible. The  Singleton bound \citep[Chapter 1, Theorem 11]{macwilliams_theory_1977} states that every $[n_q,k_q]$ linear code has minimum distance $d_\mathrm{min} \le n_q-k_q+1$. Reed-Solomon (RS) codes with parameters $\log_q K \le k_q \le n_q \le q$ achieve the Singleton bound with equality, and  the only binary code achieving the Singleton bound is the repetition code \citep[Chapter 11]{macwilliams_theory_1977} (see~\Cref{fig:cube_sphere}, middle and \Cref{sec:intuition}). 

Consider now combining RS codes with the \citet{kasarla_maximum_2022} mapping, which has $\left[{d_{\mathrm{E, min}}^{(\pi_q^{(l)})}}\right]^2 = \nicefrac{2q}{(q-1)}$. Then, the cosine similarity for this prototype construction is guaranteed to be upper bounded by
\begin{equation*}
    \langle \bm{c}_i, \bm{c}_j \rangle \le 1 - \frac{q}{q-1} \cdot \frac{n_q-k_q+1}{n_q}.
\end{equation*}
From this result, it follows that the cosine similarity of this construction becomes strictly negative if, and only if $k_q<\nicefrac{n_q}{q}+1\le 2$, where the second inequality comes from the requirement $n_q\le q$ on the length of RS codes. That is, RS codes only guarantee a strictly negative cosine similarity for $k_q=1$, given that $q\ge K$. However, in that case,  the mapping by \citet{kasarla_maximum_2022} already guarantees the optimal separation, and there is no benefit by further extending the dimensions beyond $n_q=1$ with an additional code. For $k_q=2$, $n_q=q$, and under the condition $q^2\ge K$,  we can guarantee a cosine similarity $\langle \bm{c}_i, \bm{c}_j \rangle \le 0$ for $n=(q-1)\cdot q$ dimensions. In the favourable case where $q=2^m$ and $K=2^{2m}$, the construction achieves  $\langle \bm{c}_i, \bm{c}_j \rangle \le 0$ for $n=2^{2m} - 2^m= K - \sqrt{K}$, which for $m>1$ is larger than $n\!=\!2^{2m-1} \!=\! \nicefrac{K}{2}$. $n\!=\! \nicefrac{K}{2}$ is however obtained by the RM-code-based construction as demonstrated in Lemma~\ref{lemma:orthogonal_rm_codes}. Hence, there is no benefit employing RS codes in conjunction with the mapping by \citet{kasarla_maximum_2022} compared to the RM-code-based construction. 
\subsection{Optimisation-Based Prototypes}\label{sec:optimisation_based}
In this section, we compare numerical approaches to approximately solve the non-convex and combinatorial problem \eqref{eq:optimal_hypershperical_prototypes}. Throughout, we employ projected gradient descent to deal with the non-convexity introduced by the unit norm constraint, and compare different relaxations to the combinatorial part of the problem.

\paragraph{Minimising the Average Worst-Case Similarity}
\citet{mettes_hyperspherical_2019} note that solving the combinatorial minimisation in \eqref{eq:optimal_hypershperical_prototypes} is numerically inefficient. Instead, they minimise the average maximum cosine similarity per prototype. More specifically, they define the matrix of prototypes $\bm{C} := [\bm{c}_1, \dots, \bm{c}_K] \in \mathbb{R}^{n \times K}$  describing the codebook $\mathcal{C}$ and propose the problem
\begin{equation}\label{eq:mettes_problem}
\tag{$\mathrm{P}_{\mathrm{AVG}}$}
\begin{aligned}
	\min_{\bm{C}} \quad & \frac{1}{K} \sum_{i=1}^{K} \max_{j} \bm{M}_{i,j}, \\
	\mathrm{s.t.} \quad &\bm{M} = \bm{C}^{\sf T} \bm{C} - 2 \bm{I}, \\
	\quad & \|\bm{c}_i\| = 1,
\end{aligned}
\end{equation}
where $\bm{M}_{i,j}$ denotes the $(i,j)$-th element of the matrix $\bm{M}$ and $\bm{I}$ denotes the identity matrix. Since the diagonal elements of $\bm{C}^{\sf T} \bm{C}$ are always $1$, subtracting twice the identity matrix avoids selecting these. The improvement over the original problem~\eqref{eq:optimal_hypershperical_prototypes} is that multiple prototypes are updated at each gradient step, which improves the convergence speed. However, no proof of convergence or optimality is presented. 

\paragraph{Log-Sum-Exp Relaxation}
We propose a convex relaxation to the combinatorial problem, which we show numerically that it closely approximates the converse bound in \Cref{thm:sandwich_thm}. Specifically, we propose to use the log-sum-exp approximation of the maximum. It is folklore knowledge that for $\bm{x}\in \mathbb{R}^n$ we have 
\begin{equation*}
    \max_i {\bm{x}_i} \leq \frac{1}{t} \log \left( \sum_{i=1}^n \exp(t \bm{x}_i) \right) \leq \max_i {\bm{x}_i} + \frac{\log n}{t},
\end{equation*}
for any temperature $t > 0$. Moreover, the function is convex. For a large temperature $t$, this problem approaches the original problem~\eqref{eq:optimal_hypershperical_prototypes}. By carefully choosing a scheduler for the temperature, we are able to balance the need to update multiple prototypes, and to approximate the original problem. Hence, we propose the problem 
\begin{equation}\label{eq:lse_problem}
\tag{$\mathrm{P}_{\mathrm{LSE}}$}
\begin{aligned}
	\min_{\mathcal{C}} \frac{1}{t} \log \sum_{i\ne j} \exp(t \langle \bm{c}_i, \bm{c}_j \rangle), \\
\end{aligned}
\end{equation}
which can be rewritten as a sum over the upper (or lower) triangular part of $\exp(t \bm{C}^{\sf T} \bm{C})$, excluding the diagonal.
\subsection{Computational Complexity}
In this section, we comment on the computational complexity of the prototype generation schemes in order to give a complete characterisation of the methods. We note however that the computational complexity of generating prototypes is negligible in comparison to network training.

\paragraph{Optimisation-Based Prototypes} 
The optimisation-based prototypes from \eqref{eq:lse_problem} and \eqref{eq:mettes_problem} both require $\mathcal{O}(nK^2)$ operations per gradient step, since all the $K^2$ inner products $\langle \bm{c}_i, \bm{c}_j \rangle$ of $n$-dimensional vectors need to be calculated. In practice, the wall-clock time spent on these calculations is small: Even on a laptop, the computations take on the order of seconds even for $K=1000$ and $n\approx K$.

\paragraph{Coding-Theoretic Prototypes} 
Very efficient implementations of error correcting codes exist: They are implemented and run in real time on billions of light-weight wireless devices. Since the codebooks are fixed, they need only be computed once, and can later be re-used across runs. For BCH codes, tables of their so called \textit{generator polynomials} are available, and with those, all the $K$ codewords can be enumerated quickly with a complexity of $\mathcal{O}(n K\log(K))$ operations. As an example, the generator polynomials of BCH codes are tabulated up to $n=2^{10}-1=1023$ in \citet[Appendix C]{lin_error_2004}, and up to $n=2^{16}-1 = 65 \ 535$ in the MATLAB function \texttt{bchgenpoly} \citep{matlab}. For RM codes, due to their simple structure, it is fast to generate the entire codebook, again with a complexity of $\mathcal{O}(nK\log(K))$ operations. This is done in fractions of a second even for $K=1000$ and $n\approx K$ on a laptop.

\paragraph{\citet{kasarla_maximum_2022} Prototypes} 
The prototypes from \citet{kasarla_maximum_2022} also take less than a second to generate on a laptop. However, their implementation is recursive, and therefore requires increasing the recursion limit for large $K$. Similar to the coding-theoretic prototypes, they can be pre-computed and re-used across runs.
\section{Experiments}\label{sec:experiments}
In this section, we evaluate the separation in terms of maximum cosine similarity for the considered prototype generation schemes and present numerical results on CIFAR-$100$ \citep{krizhevsky_learning_2009}. We also present supplementary results on MNIST \citep{lecun_gradient_1998} in \Cref{app:mnist}. We emphasise that our aim with these experiments is to investigate the relation between prototype separation and performance, and not necessarily to find the best performing realisations of the algorithms.
\subsection{Experimental Setup}
For optimisation-based prototypes, we follow \citet{mettes_hyperspherical_2019} and use stochastic gradient descent (SGD) with learning rate $0.1$ and momentum $0.9$ over $1 \ 000$ epochs. For the log-sum-exp prototypes, we scale the temperature linearly with epochs from $1$ to $K$. For CIFAR-$100$, we use a ResNet-$34$ backbone \citep{he_deep_2016} as implemented by \citet{kasarla_maximum_2022}, and we also use the same hyperparameters (SGD with a cosine annealing learning rate scheduler, learning rate $0.1$, momentum $0.9$, weight decay $5\times 10^{-4}$, and batch size $512$ over $200$ epochs), with standard data augmentation schemes (random $32\times 32$ crops with padding $4$, random horizontal flips with probability $\nicefrac{1}{2}$, and random rotations with up to $15^\circ$). We use cross-entropy loss on the cosine similarities $\bm{C}^{\sf T} \bm{z}$ between the prototypes $\bm{C}$ and the output $\bm{z}$ from the ResNet-$34$ backbone. At test time, classification is done via nearest-neighbour decoding, or equivalently, maximum cosine similarity decoding, by choosing the class of the nearest prototype. All our results are averaged over $5$ runs, and we use a randomised validation set (20 \% of the training set) for every run. For MNIST, we use the lightweight network proposed by \citet{simard_best_2003}, with the same hyperparameters and data augmentations as for CIFAR-$100$, except that we rotate by up to $30^\circ$ (instead of $15^\circ$) and do not flip horizontally.

\begin{figure}[t!]
	\centering
	\includegraphics[width=\linewidth]{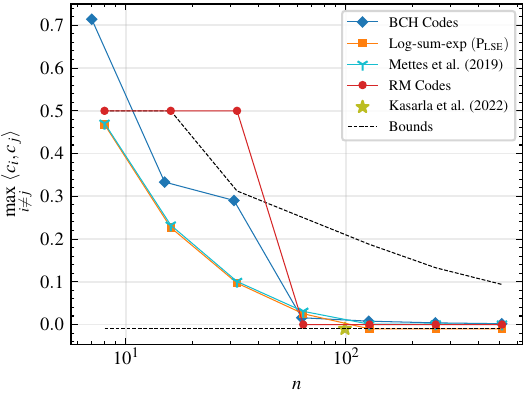}
	\caption{Maximum pairwise cosine similarity of $K=100$ prototypes in various latent space dimensions (logarithmic scale). Coding-theoretic approaches provide additional flexibility over the \citet{kasarla_maximum_2022} mapping. The optimisation-based \eqref{eq:lse_problem} prototypes achieve slightly better separation than the \eqref{eq:mettes_problem} scheme. All schemes (except for RM codes with $n=32$) fall within the achievable and converse bounds from \Cref{thm:sandwich_thm}. For large $n$, the  BCH, RM, and \eqref{eq:mettes_problem} prototypes yield no worse than othogonal prototypes. The \eqref{eq:lse_problem} and \citet{kasarla_maximum_2022} prototypes achieve the converse bound and perform therefore slightly better.}
	\label{fig:k100}
\end{figure}

Some additional remarks on prototype generation are in order. For BCH and RM codes, the mapping between class and prototype is fixed, while the optimisation-based mappings from \eqref{eq:lse_problem} and \eqref{eq:mettes_problem} randomises the assignment across different runs. Therefore, for a fair comparison for BCH and RM codes, we both average over different class to prototype mappings, as well as over a fixed class to prototype mapping.
Note that for both one-hot encoding and the \citet{kasarla_maximum_2022} mapping, the mutual cosine similarity is constant for all prototype pairs, and hence the class to prototype mapping does not matter. 
\subsection{Prototype Separation Guarantees} \label{sec:comparing_prototypes}
\begin{figure}[t]
	\centering
	\includegraphics[width=\linewidth]{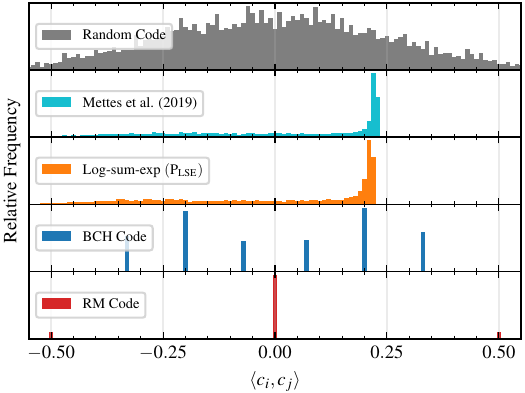}
	\caption{Cosine similarity histograms for different prototype generation schemes for $K=100$ classes in dimension $n=16$. All schemes have average cosine similarity close to $0$. Note that there are $\nicefrac{K}{2}$ RM code-based prototype pairs with cosine similarity $-1$ which have been omitted for clarity.}
	\label{fig:prototype_density}
\end{figure}
We evaluate the achieved prototype separation in terms of maximum cosine similarity over the dimension $n$  for $K \in \{ 10, 100, 1 \ 000 \}$ classes. The results for $K=100$, corresponding to the models trained on CIFAR-$100$ in \Cref{sec:experiments_cifar100}, are presented in \Cref{fig:k100} along with the achievable and converse bounds from Theorem~\ref{thm:sandwich_thm}. Plots for $K = 10$ and $K = 1 \ 000$ classes are provided in \Cref{app:additional_prototypes}. 

For $K=100$, the results confirm our theoretical analysis in \Cref{sec:binary_to_euclidian,sec:bounds}. For $n=99$, the mapping by \citet{kasarla_maximum_2022} achieves the lower bound as expected.  RM codes achieve zero worst-case cosine similarity for $n\ge64$, and the maximum cosine similarity achieved by BCH codes for $n\ge63$ is slightly above zero. That is, the code-based designs give close-to-optimal separation guarantees with only approximately half the number of dimensions. Below $n=63$, the optimisation-based schemes outperform the code-based designs, where the proposed log-sum-exp relaxation gives a slight advantage over \citet{mettes_hyperspherical_2019}. For fewer classes ($K=10$, see \Cref{app:additional_prototypes}), the optimisation-based methods outperform coding-based approaches, and solving the proposed relaxation \eqref{eq:lse_problem} instead of \eqref{eq:mettes_problem} results in a big improvement. For a large number of classes ($K=1 \ 000$, see \Cref{app:additional_prototypes}), the optimisation-based methods perform poorly, and coding-theoretic methods guarantee better separation in lower dimensions $n<K$. We therefore conclude that code-based prototypes are  beneficial  if the number of classes is large, in which case the achievable dimension compression also becomes an attractive feature. 

To provide further insights, we provide histograms of the cosine similarities for the different prototype schemes in \Cref{fig:prototype_density}. We compare the schemes for $K=100$ classes in dimension $n=16$ and, for the sake of illustration, include prototypes created uniformly at random as a baseline. As the histograms show, the optimisation moves probability mass from the right tails towards lower cosine similarity, where the log-sum-exp relaxation achieves a slightly lower maximum cosine similarity. For the coding-based methods, the cosine similarities concentrate in a few different values which can be directly calculated from the weight distribution (or weight-enumerator) function of the linear codes \citep[Chapter 2, §1]{macwilliams_theory_1977}. 
\subsection{Experiments on CIFAR-100}\label{sec:experiments_cifar100}
\begin{figure}
	\centering
	\includegraphics[width=\linewidth]{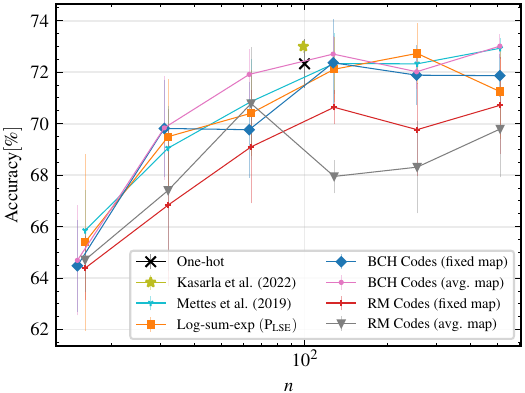}
	\caption{Top-1 accuracy results for CIFAR-$100$ for different prototype generation schemes, averaged over $5$ runs, with errorbars corresponding to one standard deviation.}
	\label{fig:cifar100_acc}
\end{figure}

We now turn to results on CIFAR-$100$, where we compare the performance of different prototype schemes. \Cref{fig:cifar100_acc} shows classification accuracy on the test set for different prototype schemes in different dimensions. We notice a thresholding effect around $n=K$, indicating that little is gained by adding dimensions, which is consistent with the observed worst-case similarity in \Cref{fig:k100}. For $n\in\{31,63,127\}$, the performance of BCH-code-based prototypes averaged over the label mapping dominates the optimization-based schemes. For $n=63$, the performance is close to the performance of one-hot encoding. The mapping by \citet{kasarla_maximum_2022} still performs best at $n=99$, which can be expected since it guarantees a slightly lower worst-case cosine similarity. 

\begin{figure*}[h]
	\centering
	\includegraphics[width=\linewidth]{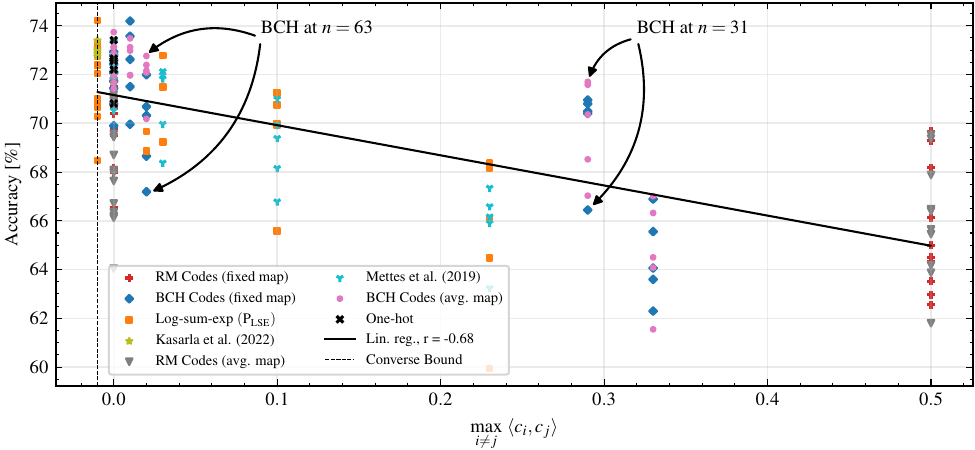}
	\caption{Comparison of accuracy on CIFAR-$100$ and maximum cosine similarity between the $K=100$ prototypes. More dissimilar prototypes are correlated with higher accuracy, but there is significant variance within, and between, models corresponding to different prototype generation schemes. Note that the same maximum similarity may correspond to different $n$, see \Cref{fig:k100}. Interestingly, some BCH codes in  lower dimension and with worse cosine similarity yield better performance than BCH codes in  higher dimension.}
	\label{fig:cifar100_sim_acc}
\end{figure*}

To illustrate the separation/accuracy tradeoff for the different prototype schemes, \Cref{fig:cifar100_sim_acc} plots the accuracy over  the maximum cosine similarity across all our trained models. Through linear regression we find that, as expected, more dissimilar prototypes tend to yield better results. However, there is a significant variance in the accuracy within models trained with the same class of prototypes, and moreover, across different methods with the same maximum similarity. Part of the variance is explained by our use of a randomised validation set. However, as the difference in BCH code performance between the fixed mapping and the average mappings, and the lower performance of RM codes show (see \Cref{fig:cifar100_acc}), the alignment of prototype similarities with \emph{semantic} similarities of classes  appears important.

In particular, the lower accuracy of RM codes in dimensions $n=64$ and $n=128$ is insightful. In these dimensions, they provide no worse than orthogonal prototypes (see \Cref{lemma:orthogonal_rm_codes}), as well as $\nicefrac{K}{2}$ prototype pairs which are diametrically opposed with cosine similarity $-1$. Investigating pairs of classes which were assigned diametrically opposed prototypes, we find several pairs with high semantic similarity, for example \texttt{leopard} and \texttt{lion}; \texttt{shrew} and \texttt{skunk}; and \texttt{seal} and \texttt{shark}. Compared to \citet{mettes_hyperspherical_2019}, who argued for the importance of the maximum and average similarity, our results indicate that additional properties beyond  maximum and average similarity are important. Hence, we argue that further investigation on incorporating the semantics in the data in the labelling of the prototypes is needed.
\section{Conclusion}\label{sec:conclusion}
In this paper, we have analysed the geometry of hyperspherical prototypical learning with tools from coding theory. Firstly, we have presented new code-based constructions to generate hyperspherical prototypes with strong minimum separation guarantees (in terms of worst-case cosine similarity). Secondly, we fully characterised the worst-case cosine similarity of these prototypes (in terms of achievable and converse bounds). Our prototypes are flexible and near-optimal in low dimension, thereby enabling a trade-off between dimension and separation for a given number of classes. Our experimental results furthermore indicate that the classification accuracy does not only depend on the worst-case separation of  prototypes, but also depends on the mapping from class labels to prototypes. We thus conclude that the alignment of semantic similarity with prototype separation is an important problem for further investigation. Additionally, the impact of prototype distance on prototype-based self-supervised learning schemes is also an important future consideration.

\iftoggle{fullversion}{
\section*{Acknowledgements}
This work was funded in part by the Swedish Research Council (VR) through grant agreements 2019-03606 and 2021-05266. The computations were enabled by resources provided by the National Academic Infrastructure for Supercomputing in Sweden (NAISS) at Chalmers Centre for Computational Science and Engineering (C3SE), partially funded by the Swedish Research Council through grant agreement 2022-06725. 
}
%

\bibliography{references}
\bibliographystyle{icml2024}

\newpage
\appendix
\onecolumn
\section{Separation Guarantees for $K=10$ and $K=1 \ 000$ Prototypes}\label{app:additional_prototypes}
In this section, we provide additional results for prototype generation schemes for $K=10$ and $K1 \ 000$ prototypes, see \Cref{fig:k10,fig:k1000}. For $K=10$ prototypes, the optimisation-based approaches work well: in particular, solving \eqref{eq:lse_problem} provides no worse than orthogonal prototypes in dimension $n=8$, and optimally separated prototypes in $n=16$. However, the improvement over one-hot encoding and the \citet{kasarla_maximum_2022} mapping is small. On the other hand, for $K=1 \ 000$ prototypes, the coding-theoretic approaches can guarantee no worse than orthogonal (and therefore near-optimal) separation in $n\approx \nicefrac{K}{2}$, while optimisation-based prototypes require high dimension to give approximately orthogonal prototypes. 
\begin{figure}[H]
	\centering
	\includegraphics[width=0.5\linewidth]{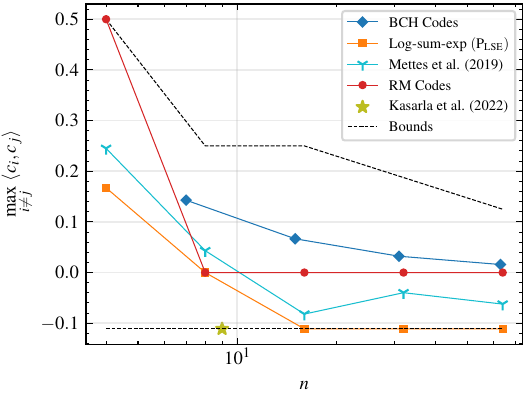}
	\caption{Maximum pairwise cosine similarity of $K=10$ prototypes in various latent space dimensions (logarithmic scale). The optimisation-based \eqref{eq:lse_problem} prototypes achieve better separation than the \eqref{eq:mettes_problem} scheme. One-hot encoding and the \citet{kasarla_maximum_2022} are competitive in this regime. All schemes fall within the achievable and converse bounds from \Cref{thm:sandwich_thm}. For large $n$, the the BCH and RM prototypes yield no worse than othogonal prototypes. The \eqref{eq:lse_problem} and \citet{kasarla_maximum_2022} prototypes achieve the converse bound, and are therefore better. The \eqref{eq:mettes_problem} prototypes are more than orthogonal, but do not achieve the converse bound.}
	\label{fig:k10}
\end{figure}
\begin{figure}[H]
	\centering
	\includegraphics[width=0.5\linewidth]{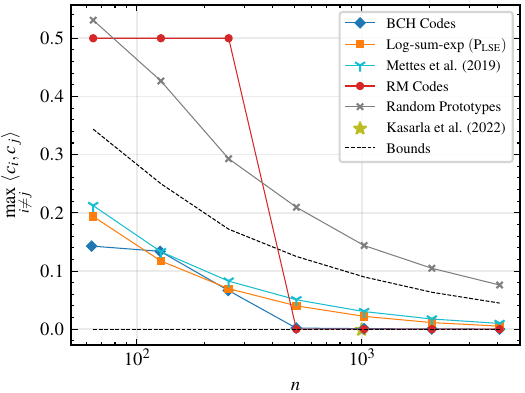}
	\caption{Maximum pairwise cosine similarity of $K=1 \ 000$ prototypes in various latent space dimensions (logarithmic scale). Coding-theoretic approaches give near-optimal separation in low dimension $n\ge511$, and offer better flexibility than the \citet{kasarla_maximum_2022} mapping. The optimisation schemes require high dimension to be competitive, while the coding-theoretic prototypes are no worse than orthogonal in low dimension. In very high dimensions, even uniformly random prototypes are near-orthogonal.}
	\label{fig:k1000}
\end{figure}
\newpage
\section{Cosine Similarity Histograms for $K=10$ and $K=1 \ 000$ prototypes}
In this section, we show cosine similarity histograms for $K=10$ and $K=1 \ 000$ prototypes, see \Cref{fig:density_k10,fig:density_k1000}. The figures illustrate that optimisation-based methods are effective for a small number of prototypes, where the density can be moved around effectively to create good maximum cosine similarity properties. The coding-theoretic approaches concentrate the density in a few cosine similarities, which is beneficial for a large number of prototypes $K$, where they outperform the optimisation-based approaches in lower dimensions $n$.
\begin{figure}[H]
	\centering
	\includegraphics[width=0.5\linewidth]{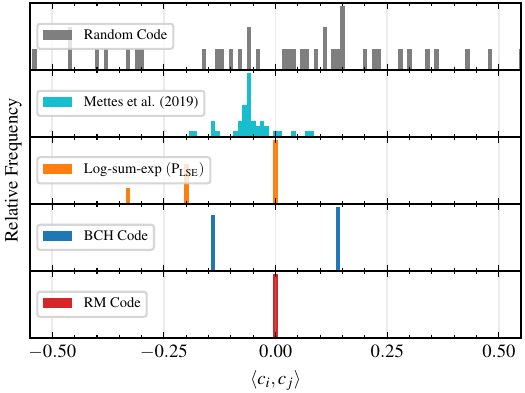}
	\caption{Cosine similarity histogram for different prototype generation schemes for $K=10$ classes and dimension $n=8$. For a small number of classes, the optimisation-based approaches are more efficient in moving probability mass to lower similarities. Note that there is a low but non-zero density for the RM code at similarity $-1$, which has been omitted for clarity.}
 \label{fig:density_k10}
\end{figure}
\begin{figure}[H]
	\centering
	\includegraphics[width=0.5\linewidth]{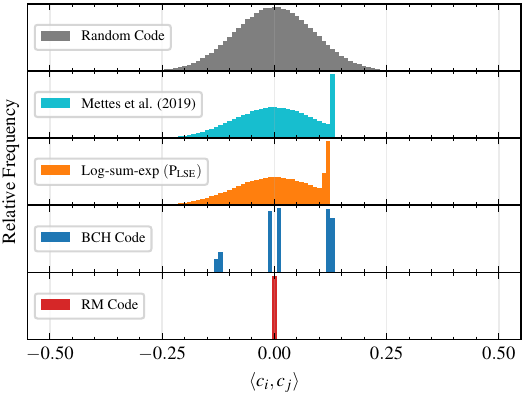}
	\caption{Cosine similarity histogram for different prototype generation schemes for $K=1 \ 000$ classes and dimension $n=128$. The coding-based approaches concentrate density at a few similarities, which outperforms optimisation-based methods for a large number of classes. Note that for the RM code there is non-zero density at similarity $-0.5$ and $0.5$ which are too small to be visible, and that there is density at similarity $-1$ which has been omitted for clarity.}
	\label{fig:density_k1000}
\end{figure}
\newpage
\section{Results on MNIST}\label{app:mnist}
In this section, we provide results on MNIST, similar to the ones on CIFAR-$100$, in \Cref{fig:mnist_acc,fig:mnist_sim_acc}. We notice that while all the schemes perform well on MNIST, we are able to beat one-hot encoding and the \citet{kasarla_maximum_2022} mapping using prototypes obtained from solving \eqref{eq:lse_problem} in lower dimension. We also find that a smaller maximum cosine similarity is correlated with better performance, although the correlation is weaker here than for CIFAR-$100$. 

Similar to CIFAR-$100$, we notice a high variance. In particular, notice the averaged mappings outperforms the fixed mappings for both BCH and RM codes. This again indicates that the semantic class to prototype mapping is important. Again, investigating mappings for RM codes for $n=8$, where the prototypes are no worse than orthogonal, we find the following diametrically opposed pairs in the first trained model. For the fixed mapping, the pairs \texttt{4} and \texttt{5}; and \texttt{8} and \texttt{9} were diametrically opposed. These pairs are semantically similar, especially when hand-written. A randomised class to prototype assignment has a chance of avoiding these semantically similar pairings.

\begin{figure}[H]
	\centering
	\includegraphics[width=0.5\linewidth]{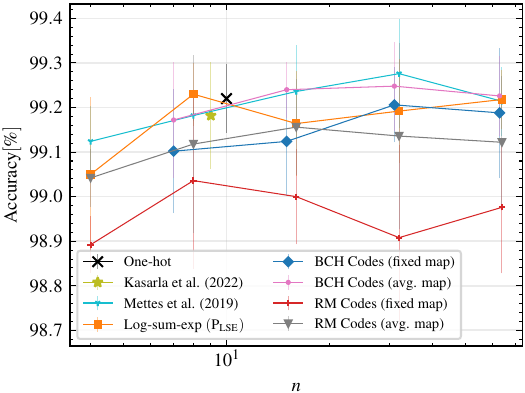}
	\caption{Top-1 accuracy results for MNIST over different prototype generation schemes, averaged over 5 runs, with error bars corresponding to one standard deviation.}
	\label{fig:mnist_acc}
\end{figure}
\begin{figure}[H]
	\centering
	\includegraphics[width=.9\linewidth]{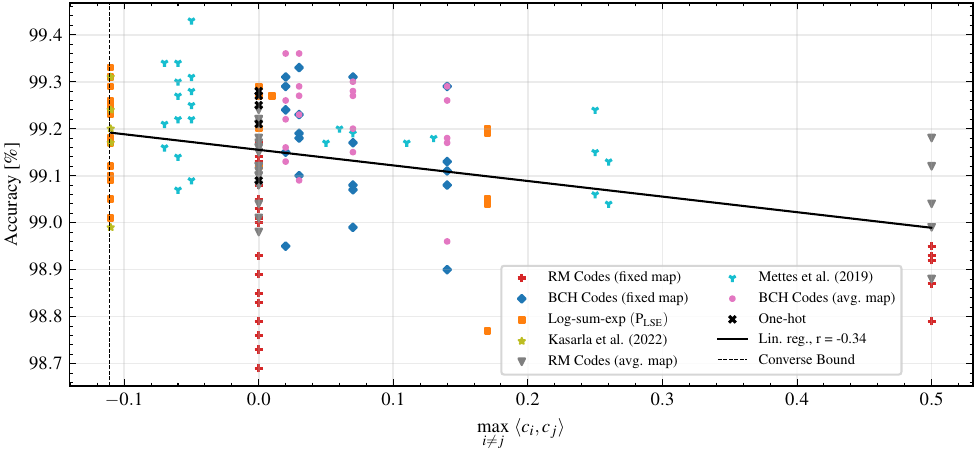}
	\caption{Comparison of accuracy on MNIST and maximum cosine similarity between the $K=10$ prototypes. Higher accuracy is correlated with smaller maximum cosine similarity. However, there is a large variance within, and between, different prototype generation scheme. Note that the same maximum similarity may correspond to different $n$, see \Cref{fig:k10}.}
	\label{fig:mnist_sim_acc}
\end{figure}
%

\end{document}